\documentclass{article}
\usepackage{arxiv}
\usepackage[caption=false]{subfig}
\usepackage{graphics}

\usepackage{amsmath}
\usepackage{import}
\usepackage{pdfpages}
\usepackage{amsthm}
\usepackage{amssymb}
\usepackage{ulem}
\usepackage{mathtools}

\usepackage{algorithm,float}
\usepackage[hidelinks]{hyperref}
\usepackage{comment}
\usepackage{url}
\usepackage{xcolor}
\usepackage{soul}

\usepackage[noend]{algpseudocode}
\newcommand{\ie}{\textit{i.e.}}
\newcommand{\eg}{\textit{e.g.}}
\newcommand{\cf}{\textit{cf.}}

\newcommand{\E}{\mathbb{E}}
\newcommand{\Pp}{\mathbb{P}}

\newcommand{\PW}[1]{{\color{black}#1}}
\newcommand{\KK}[1]{{\color{black}#1}}
\newcommand{\CR}[1]{{\color{black}#1}}

\newtheorem{theorem}{Theorem}
\newtheorem{lemma}{Lemma}

\newtheorem{remark}{Remark}

\newcommand{\sfunction}[1]{\textsf{\textsc{#1}}}
\algrenewcommand\algorithmicforall{\textbf{\PW{for each}}}
\providecommand{\keywords}[1]
{
  \small	
  \textbf{\textit{Keywords ---}} #1
}

\usepackage{color}

\DeclareMathOperator*{\argmin}{argmin}

\title{Decentralized diffusion-based learning\\ under non-parametric limited prior knowledge}
\author{Pawe\l{} Wachel\\
	Department of Control Systems and Mechatronics\\
	Wroc\l{}aw University of Science and Technology\\
	Wroc\l{}aw, Poland \\
	\texttt{pawel.wachel@pwr.edu.pl} \\
	\And
	Krzysztof Kowalczyk\\
	Department of Control Systems and Mechatronics\\
	Wroc\l{}aw University of Science and Technology\\
	Wroc\l{}aw, Poland \\
	\texttt{krzysztof.kowalczyk@pwr.edu.pl} \\
	\And
Cristian R. Rojas\\
	School of Electrical Engineering and Computer Science\\
	KTH Royal Institute of Technology\\
	Stockholm, Sweden \\
	\texttt{crro@kth.se} 
}
\date{}

\begin{document}

\maketitle

\begin{abstract}
We study the problem of diffusion-based network learning of a nonlinear phenomenon, $m$, from local agents' measurements collected in a noisy environment. For a decentralized network and information spreading merely between directly neighboring nodes, we propose a non-parametric learning algorithm, that avoids raw data exchange and requires only mild \textit{a priori} knowledge about $m$. Non-asymptotic estimation error bounds are derived for the proposed method. Its potential applications are illustrated through simulation experiments.
\end{abstract}

\section{Introduction}
The field of decentralized and distributed learning fits in with the area of modern Internet-of-Things (IoT)  and wireless sensor networks (WSN) applications. Due to technological advances and functional advantages related to robustness and scalability \cite{vieira2003survey}, decentralized and distributed techniques are becoming more widespread in industry and are the subject of ongoing scientific research. Among various goals specific for learning and inference in decentralized networks, like learning linear modules \cite{ramaswamy2021learning}, distributed economic dispatch \cite{guo2015distributed} or target tracking \cite{martinovic2022cooperative}, one can point out estimation tasks as in \cite{bertrand2011consensus} or \cite{cattivelli2009diffusion}. In this scenario, sensors (or agents) are scattered around a given area and collect data about an unknown phenomenon, \PW{modelled as a nonlinear function $m\colon \mathbb{R} \to \mathbb{R}$}. Due to potential communication restrictions and the lack of dedicated fusion centers, agents may rely only on their local/private measurements and available network information. Following \cite{modalavalasa2021review} and \cite{he2020distributed} we begin with brief summary of a few well-known strategies.

In the incremental approach \cite{lopes2007incremental,bogdanovic2014distributed}, local parameter estimates are transferred through the network in cycles, which leads to significant communication cost reduction. Nevertheless, determining network cycles (Hamiltonian paths) is known to be NP-Hard \cite{rabbat2005quantized}, which leads to high computational expenses whenever the network architecture changes, \textit{e.g.},~due to node failure or agent movement. However, the performance of such algorithms may be close to their centralized counterparts in the case of static, energy-efficient, and fail-safe networks~\cite{bogdanovic2014distributed}.

Consensus- and diffusion-based approaches provide more general solutions in which data samples are broadcast in the agents' local neighborhoods. However, some of these methods impose implementation restrictions in real-time problems; see \eg~discussion in \cite{tu2012diffusion}. In this context, improvements have been made \eg~in \cite{schizas2009distributed}; see also \cite{lopes2008diffusion, sayed2007distributed}. It is worth pointing out that, as shown in \cite{tu2012diffusion}, diffusion strategies perform better than consensus in the fields of network stability and information spreading. \KK{In the literature, we may also find a detailed discussion of well-known in-network learning obstacles, for instance, related to asynchronous learning~\cite{sayed2018asynchronous,zhang2014fusion}, flexible user-defined measurement and communication periods~\cite{jiang2012probability,yang2014stochastic}, or network topology changes~\cite{paradis2007survey}. }

Most of the above-mentioned methods require parametric \textit{a priori} information, that allows reformulating the considered problem as the estimation of a finite set of parameters.
The approach considered in this paper weakens these requirements and admits significantly limited a priori knowledge of a non-parametric form. In fact, we assume only that the latent nonlinear phenomenon of interest is a Lipschitz continuous function. 
Furthermore, we do not employ gradient descent routines and avoid the related network stability issues \cite{tu2012diffusion}. At the same time, the proposed algorithm tends to compensate for low-quality local measurements of agents by exploiting incoming information from the network.

The main contributions of the paper are as follows:
\begin{itemize}
    \item We propose a new decentralized learning algorithm for modeling latent nonlinear phenomena observed locally by the agents in a sensor network.
    \item The accuracy of the proposed algorithm is formally investigated and error bounds are derived for user-specified significance levels via concentration inequalities.
    \item Formal properties of the algorithm are investigated under weak assumptions imposed on explanatory data, measurement noises and the nonlinear phenomenon of interest.
    \item The theoretical analysis of the algorithm is related to its practical aspects, as the error bounds hold for any finite (non-asymptotic) number of observations.
    \item The algorithm and the corresponding error bounds are simple and easy to interpret. Thus, they can be implemented and evaluated in agents with limited computing resources. 
    \item The algorithm preserves data privacy in the sense that nodes do not share their local (private) output measurements; only raw pre-estimates are exchanged in the network. 
    \item The algorithm is robust against the variability of measurement noise levels among the network nodes. In particular, it tends to compensate for low-quality local observations with available online network information.
\end{itemize}
The paper is organized as follows. In Sections~\ref{Sec:PF} and~\ref{Sec:Alg} we formally state the considered network learning task, introduce \textit{model learning} and \textit{model exploiting} routines, and discuss formal properties of the method. The results of numerical experiments are described in Section~\ref{Sec:Num} and concluding remarks are collected in Section~\ref{Sec:Con}.  

\section{Problem formulation}\label{Sec:PF}
We study a decentralized network of $M$ agents inferring under limited prior knowledge about an unknown nonlinear phenomenon of interest. In the considered problem, we model the network as a connected and undirected graph $\mathcal{G}=(\mathcal{M},\mathcal{E})$ with $\mathcal{M}=\{1,2,\dots,M\}$ nodes and a set of unweighted edges $\mathcal{E}$. Two nodes $i,j\in\mathcal{M}$ can exchange information if and only if they are directly connected, \textit{i.e.}, if $\{i,j\}\in\mathcal{E}$. For any node $i\in\mathcal{M}$, its neighborhood is the set $\mathcal{N}_i=\{j\colon \{i,j\}\in\mathcal{E}\}$. 

We consider a network learning scenario in which at every time step $t\in \mathbb{N}$ each agent $k\in\mathcal{M}$ acquires explanatory data $\xi _{k,t}^{loc}\in \mathbb{R}$, and observes  outcomes of an underlying unknown nonlinear phenomenon $m\colon \mathcal{D} \subset \mathbb{R}\rightarrow\mathbb{R}$ as noisy measurements $y_{k,t}^{loc}$, where
\begin{equation} \label{eq:measurements}
y_{k,t}^{loc}=m( \xi _{k,t}^{loc}) +\eta_{k,t}^{loc},\quad k\in{\mathcal{M}},\quad t\in \mathbb{N},
\end{equation}
and $\eta_{k,t}^{loc}$ stands for a sensor noise/inaccuracy. The unknown map $m$ is the same for the entire network, although the measurements of individual agents may have only a \textit{local} character\footnote{For ease of presentation, in the paper we use superscripts `\textit{loc}', `\textit{acq}', `\textit{req}' and `\textit{shr}' as abbreviations of terms \textit{local}, \textit{acquired}, \KK{\textit{requested}} and \textit{shared}, respectively.}, \textit{i.e.} for any agent $k\in\mathcal{M}$ the corresponding data set $\{ ( \xi_{k,n}^{loc},y_{k,n}^{loc}) ,n=1,2,\dots,t\}$ allows to infer $m$ only within some local subset $\mathcal{D}_k$ of domain $\mathcal{D}$; $\xi_{k,n}\in\mathcal{D}_{k}$, \textit{cf.} \KK{Fig. \ref{fig:NetDiagram}}.

While the agents may have a limited view of the phenomenon under study, our goal is to design a decentralized network learning algorithm, enabling them to expand their local (private) inference capabilities into a broader range $\mathcal{D}_1\cup\mathcal{D}_2\cup\dots\cup\mathcal{D}_M$. However, since all the agents can communicate merely with their direct neighbors, the problem at hand is of an information-diffusion nature.   
\begin{figure}[ht]
    \centering
    \includegraphics[width=0.8\textwidth]{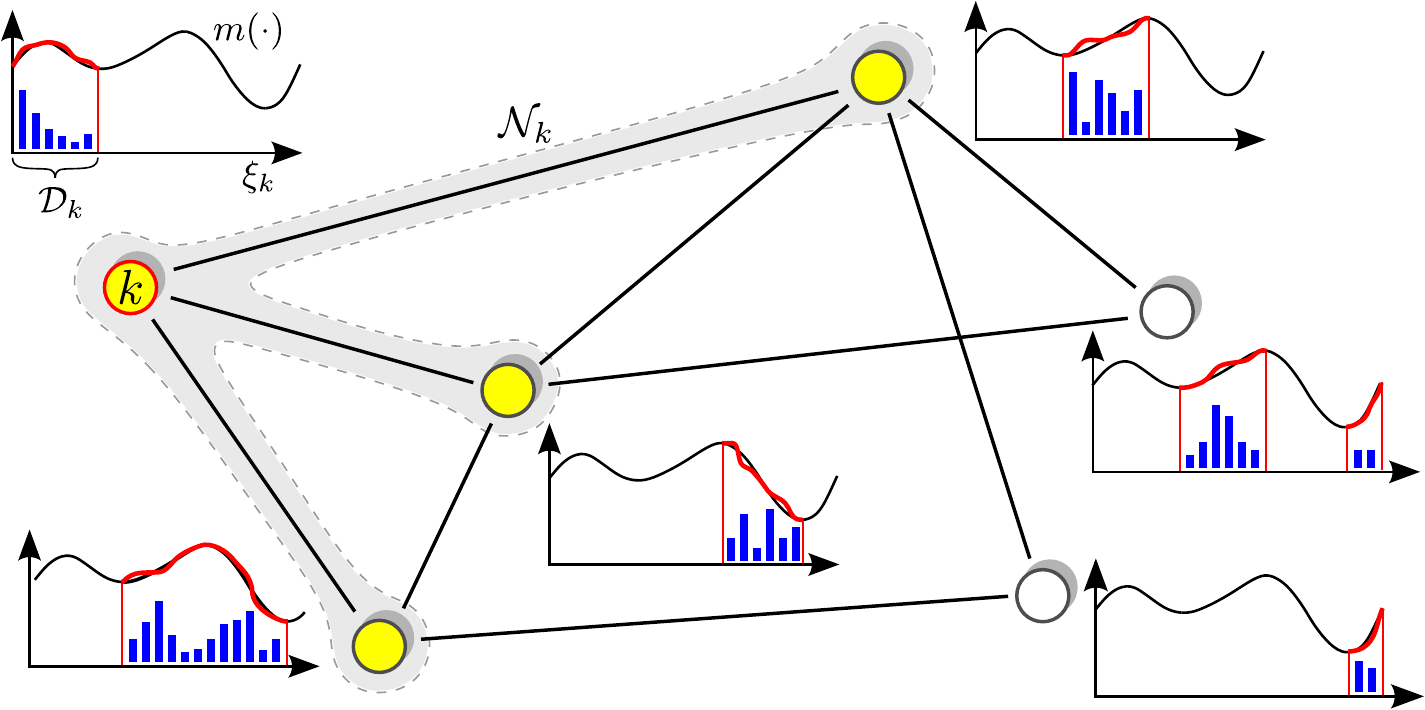}
    \caption{A decentralized network of agents with histograms of local explanatory data. The neighborhood $\mathcal{N}_k$ of agent $k$ has been highlighted.}
    \label{fig:NetDiagram}
\end{figure}

To ensure the generality of the proposed approach, we assume very mild \textit{a priori} knowledge about the latent phenomenon. In fact, it is only required that $m$ is Lipschitz continuous. Hence, the problem at hand is non-parametric, and parametric estimation techniques may lead to inconsistent results.

\textit{Assumption} $1$. The latent phenomenon of interest, $m\colon \mathcal{D} \subset \mathbb{R}\rightarrow \mathbb{R}$, is a Lipschitz continuous mapping, \textit{i.e.}
    \begin{equation}
        \lvert m(\xi)-m(\xi')\rvert\leq L\vert \xi-\xi'\vert,\quad \forall\;\xi,\xi'\in\mathcal{D},
    \end{equation}
    with a known constant $L<\infty$.
    
In general, the goal of every agent $k\in\mathcal{M}$ is to learn $m$ from its local and acquired (networked) information. We characterize this scenario in detail under weak assumptions about the agents' measurements and their corresponding uncertainties.
    
\textit{Assumption} $2$. For any agent $k\in\mathcal{M}$ in the network $\mathcal{G}$, the local explanatory sequence $\{\xi_{k,t}^{loc}\colon t\in\mathbb{N}\}$ is an arbitrary stochastic process $\xi_{k,t}^{loc}\in\mathcal{D}_{k}\subset\mathcal{D}$.

\textit{Assumption} $3$. For any agent $k\in\mathcal{M}$ in the network $\mathcal{G}$, the disturbance $\{\eta_{k,t}^{loc}\colon t\in\mathbb{N}\}$, is a sub-Gaussian stochastic process, 
\PW{in the sense that} there exists a $\sigma_k>0$ such that for every $\gamma\in\mathbb{R}$ and $t\in \mathbb{N}$,
\begin{equation}\label{eq:A_eta}
\E\{ \exp ( \gamma \eta _{k,t}^{loc})\PW{|\eta_{k,1}^{loc}, \dots, \eta_{k,t-1}^{loc}, \xi_{k,1}^{loc}, \dots, \xi_{k,t}^{loc}}\} \leq \exp
\left( \frac{\gamma ^{2}\sigma_k ^{2}}{2}\right).
\end{equation}
The above assumptions require some remarks. Observe first that the class of local explanatory sequences is very general and does not impose any limitations on the functional dependence between $\xi_{k,t}^{loc}$ and $\xi_{l,t}^{loc}$ for $k,l\in\mathcal{M}$. \PW{Regarding condition \eqref{eq:A_eta}, it can be proven it implies that $E\{\eta _{k,t}^{loc}|\eta _{k,1}^{loc},\dots,\eta _{k,t-1}^{loc}, \xi_{k,1}^{loc},\dots,\allowbreak \xi_{k,t}^{loc}\}=0$ and $E\{(\eta_{k,t}^{loc})^2|\eta _{k,1}^{loc}, \dots, \eta_{k,t-1}^{loc}, \xi_{k,1}^{loc}, \dots, \allowbreak \xi_{k,t}^{loc}\}\leq\sigma_k^2$, \cf~\cite[Exercise~2.5]{wainwright2019high}. Furthermore,} the variance proxies $\sigma_k$ can be different for different nodes of the network. \PW{The condition \eqref{eq:A_eta} is fulfilled in particular by an \textit{i.i.d.}~zero-mean Gaussian sequence, independent of the process $\{\xi_{k,t}^{loc}\}$.} \PW{Finally}, if $\eta _{k,t}^{loc}\in[a,b]$ a.s., for some $a,b\in\mathbb{R}$, then $\sigma_k^2$ \CR{can be taken as} $(b-a)^2/4$ by Hoeffding's Lemma \PW{(\CR{even though} any value of $\sigma_k^2$ larger than $(b-a)^2/4$ would \CR{also} satisfy \eqref{eq:A_eta} in this case).}

\section{Non-parametric decentralized learning}\label{Sec:Alg}
According to the problem statement, every agent in the network has the ability to collect its local (private) observations $\{ (\xi_{k,n}^{loc},y_{k,n}^{loc}) ,n=1,2,\dots,t\}$, allowing for merely a \textit{local} estimation of mapping $m$. In the sequel, we focus on the non-parametric, numerically efficient, Nadaraya--Watson estimator
\begin{equation}\label{eq:NW1}
  \begin{aligned}
\hat{\mu}_{k,t}(x):=&\sum_{n=1}^{t}\frac{K_h(x,\xi_{k,n}^{loc})}{\kappa_{k,t}(x)}y_{k,n}^{loc},\\
\quad\kappa _{k,t}(x):=&\kappa_{k,t}(x,h)=\sum_{n=1}^{t}K_h(x,\xi_{k,n}^{loc}),
 \end{aligned}
\end{equation}
where $k$ stands for the agent index, $K_h(x,\xi):=K((x-\xi)/h)$, and where $K$ with $h$ are the kernel function and bandwidth parameter, respectively.

In general, the pointwise consistency of $\hat{\mu}_{k,t}(x)$ can be proven in various probabilistic senses under mild requirements regarding $m(x)$, as long as the observations `tend' to concentrate in a neighborhood of $x$. Such results, however, often have an asymptotic flavor and cannot be readily adapted to finite data set tasks. Hence, for the considered decentralized non-parametric learning, we develop a non-asymptotic bound. To this end, consider the following assumption:

\textit{Assumption} $4$. The kernel $K\colon \mathbb{R}\rightarrow\mathbb{R}$ is a non-negative and bounded function such that $0 \leq K(v)\leq1$ for all $v\in\mathbb{R}$. Also, $K(v)=0$ for all $|v|>1$.

\begin{lemma}\label{L:local}
Let Assumptions 1--4 be in force.
Consider agent $k\in\mathcal{M}$ in the network $\mathcal{G}$, equipped with the estimator $\hat{\mu}_{k,t}$, and fix a bandwidth $h$. \CR{Let } $x \in \mathcal{D}_k$ \CR{be a fixed deterministic number, or in general a function of $\eta_{k,1}^{loc}, \dots, \eta_{k,t-1}^{loc}, \xi_{k,1}^{loc}, \dots, \xi_{k,t}^{loc}$} \PW{(\eg, $x=\xi_{k,t}^{loc}$)}. Then, for every $ 0<\delta <1$, with probability at least $1-\delta $, \CR{if $\kappa _{k,t}(x)\neq 0$},
\begin{align}
\vert \hat{\mu}_{k,t}(x) -m(x)
\vert \leq \beta_{k,t}(x),\\ \label{eq:main_bound}
\quad\text{ where}\quad
\beta_{k,t}(x) := L h+2\sigma_k \frac{\alpha _{k,t}(x ,\delta )}
{\kappa _{k,t}(x)}
\end{align}
and 
\begin{equation*}
\alpha _{k,t}( x ,\delta ) :=
\begin{cases} 
\sqrt{\log ( \sqrt{2}/\delta )}, & \textnormal{for } \kappa _{k,t}(x) \leq 1 \vspace{3pt}\\ 
\sqrt{\kappa _{k,t}(x) \log ( \sqrt{1+\kappa _{k,t}(x)}/\delta )}, & \textnormal{for } \kappa _{k,t}(x) > 1.
\end{cases}
\end{equation*}
\end{lemma}
\begin{proof}
    See the Appendix.
\end{proof}
The estimation error bound in the lemma above has a direct practical character since $\beta_{k,t}(x)$ can be evaluated for any finite number of measurements if $\sigma_k$ and $L$ are known or can be upper-bounded by known constants. \KK{Note that, unlike the common assumption used in Nadaraya-Watson estimation (in particular, symmetry or unit integral of the kernel), the kernel, $K$, has to satisfy only the mild requirements in Assumption $4$. However, since the confidence bound $\beta_{k,t}$ strongly depends on $\kappa _{k,t}$, we can choose the kernel $K$ and the bandwidth parameter $h$ to minimize $\beta_{k,t}$. Nevertheless, we leave these aspects for future research.  For simplicity, in the implementation presented in Section 4 we use the classical box kernel.}

Based on bound \eqref{eq:main_bound}, we can proceed to describe our proposed network learning algorithm, presented in Algorithm~\ref{Alg:1}. From now on we focus on an arbitrary but fixed agent $k\in\mathcal{M}$ in the network. Clearly, for a given local explanatory sample $\xi_{k,t}^{loc}$, one can form a corresponding \textit{local} tuple $T(\xi_{k,t}^{loc}):=(\xi_{k,t}^{loc},\hat{\mu}_{k,t}^{loc},\beta_{k,t}^{loc})$ with $\hat{\mu}_{k,t}^{loc}=\hat{\mu}_{k,t}(\xi_{k,t}^{loc})$ and $\beta_{k,t}^{loc}=\beta_{k,t}(\xi_{k,t}^{loc})$. We denote the set of all such tuples (collected by the considered agent \KK{$k$}) as $\mathbb{T}^{loc}_{\KK{k}}$. By analogy, let $T(\xi_{l,t}^{acq}):=(\xi_{l,t}^{acq},\hat{\mu}_{l,t}^{acq},\beta_{l,t}^{acq})$ be the tuple \textit{acquired} by agent $k$ from node $l\in\mathcal{N}_k$. The set of all such tuples (for agent $k$) is denoted as $\mathbb{T}^{acq}_{\KK{k}}$.
\KK{
\begin{remark}
Algorithms~1 and~2 present the proposed method for non-parametric decentralized learning. They are constructed based on the following commands:
\begin{itemize}
    \setlength\itemsep{0.1em}
    \item $\sfunction{Get}$ -- get the local explanatory sample and the corresponding noisy measurement of the phenomenon outcome, $(\xi_{k,t}^{loc},y_{k,t}^{loc})$, for $t=1,2,\dots$ 
    \item $\sfunction{Evaluate}$ -- evaluate the estimator value and/or its corresponding error bound according to eq. $(4)$ and/or Lemma $1$, respectively.
    \item $\sfunction{Append}$ -- append a new tuple to the set of \textit{local} or \textit{acquired} tuples. In general, the $\sfunction{Append}$ command may contain an update mechanism that \textit{eliminates} less accurate tuples from a given collection when more accurate data is available (a simple Lipschitz-based routine for this is applied in the implementation used in Section 4).  
    \item $\sfunction{IsAvailable}$ -- check whether tuple $T(\xi^{acq}_{l,t-1})$ or the requested sample, $\xi^{req}_{l,t-1}$, has been obtained from the neighbour $l\in\mathcal{N}_k.$
    \item $\sfunction{Send}$ -- send a given tuple to the requesting node $l\in\mathcal{N}_k$.
    \item $\sfunction{Select}$ -- select an argument to be requested by the agent. In general, the argument can be deterministic or any (measurable) function of $\eta_{k,1}^{loc}, \dots,\allowbreak \eta_{k,t-1}^{loc}, \xi_{k,1}^{loc}, \dots, \xi_{k,t}^{loc}$ at time $t$. In the implementation used in Section 4 it is sampled from a uniform distribution on $\mathcal{D}$. 
    \item $\sfunction{Broadcast}$ -- send a request to the neighbours.
    \end{itemize}
\end{remark}
}

\CR{Please notice that a link to the MATLAB code implementation used in Section 4 is provided in the paper.}

\begin{algorithm}[H]
\caption{Non-parametric diffusion-based \textbf{model learning} \Comment{Agent $k$}} \label{Alg:1}

\begin{algorithmic}[1]
\State \textbf{input:} $\{T(\xi_{l,t}^{acq})\colon l\in\mathcal{N}_k\}$, $\{\xi_{l,t}^{req}\colon l\in\mathcal{N}_k\}, \delta, L, \sigma_k $
\State {\textbf{initialization:}} $\mathbb{T}^{loc}_{\KK{k}}=\varnothing, \mathbb{T}^{acq}_{\KK{k}}=\varnothing$
\For{$t=1,2,\dots$}
    \State $\sfunction{Get }(\xi_{k,t}^{loc},y_{k,t}^{loc})$ \Comment{Get local measurement}
    \State $\sfunction{Evaluate }\hat{\mu}_{k,t}^{loc}=\hat{\mu}_{k,t}(\xi_{k,t}^{loc}),\beta_{k,t}^{loc}=\beta_{k,t}(\xi_{k,t}^{loc})$ as in Eqs.~\eqref{eq:NW1}--\eqref{eq:main_bound} 
    \KK{\State $\sfunction{Append }$$T(\xi_{k,t}^{loc}):=(\xi_{k,t}^{loc},\hat{\mu}_{k,t}^{loc},\beta_{k,t}^{loc})$ to $\mathbb{T}^{loc}_{\KK{k}}$ }
    \ForAll {$l\in\mathcal{N}_k$} \Comment{Acquire shared tuples}
        \If  {$\sfunction{IsAvailable}($\KK{$T(\xi_{l,t-1}^{acq})$}$)$,}
           \KK{ \State $\sfunction{Append }$ $T(\xi_{l,t-1}^{acq})$
            to $\mathbb{T}^{acq}_{\KK{k}}$
            \EndIf}

        \State \textbf{end}    
        \If  {$\sfunction{IsAvailable}($\KK{$\xi_{l,t-1}^{req}$}$)$,} \Comment{Send requested tuples}
            \State $\xi_{l,t}^{shr}\gets\argmin_{\xi}\vert\xi-\xi_{l,t-1}^{req}\vert$  s.t.~$T(\xi)\in\{\mathbb{T}^{loc}_{\KK{k}}\cup\mathbb{T}^{acq}_{\KK{k}}\}$
            \State $\sfunction{Send }T(\xi_{l,t}^{shr})$ to node $l$
            \EndIf
        \State \textbf{end}
    \EndFor
    \State \textbf{end}
\State $\sfunction{Select }\xi_{k,t}^{req}$
\State $\sfunction{Broadcast }\xi_{k,t}^{req}$\Comment{Broadcast request to all agents $l\in\mathcal{N}_k$}
\EndFor
\State \textbf{end}
\end{algorithmic} 
\end{algorithm}

\begin{remark}
To apply estimator \eqref{eq:NW1}, one has to pick the bandwidth $h>0$ corresponding to the selectivity of the kernel $K$. Note that in the considered approach it is reasonable to pick $h$ so that it minimizes the upper bound in \eqref{eq:main_bound}. For \CR{a fixed} $x\in\mathcal{D}$, an appropriate bandwidth can be found numerically, \eg, using the Golden Section Search algorithm, since the two terms in $\beta_{k,t}(x)$ are monotone with respect to $h$ (in particular, for any set of explanatory data $\{\xi_{k,n}^{loc}\}$, $\kappa_{k,t}(x)$ grows as $h$ increases).
\end{remark}

In the considered problem setup, agents exchange tuples with their direct neighbors according to the network topology $\mathcal{G}$.  The tuple diffusion protocol is a part of the proposed non-parametric learning method described in Algorithm $1$. For a given agent $k$, at time step $t$ the procedure has two inputs: a set of acquired tuples $\{T(\xi_{l,t}^{acq})\colon l\in\mathcal{N}_k\}$, and a set of requested arguments $\{\xi_{l,t}^{req}\colon l\in\mathcal{N}_k\}$.

\begin{remark}
One of the key parts of Algorithm 1 is the tuple request mechanism `$\sfunction{Select }\xi_{k,t}^{req}$'. As shown in Theorem \ref{Th:main}, the accuracy of the final estimate, $\hat{m}_{k,t}$, at any given point $x$, directly depends on the tuples acquired from the agent's neighbors. Hence, the selection of $\xi_{k,t}^{req}$ should depend on those arguments $x\in\mathcal{D}$ that are particularly important for the corresponding agents in the network. 
\end{remark}

The actual inference about the phenomenon under investigation is performed by agent $k$ according to Algorithm $2$. For any requested $x\in\mathcal{D}$, not necessarily from $\mathcal{D}_k$, data sets $\mathbb{T}^{loc}_{\KK{k}},\mathbb{T}^{acq}_{\KK{k}}$ from Algorithm $1$ are used to evaluate both the estimate $\hat{m}_{k,t}(x)$ and its corresponding error bound $\beta_{k,t}(x)$. 
For a given $\delta\in(0,1)$ the algorithm compares the local prediction of agent $k$ with the acquired network information and returns a non-parametric estimate of $m(x)$ having the tightest available error upper bound with confidence level $1-\delta$.

\begin{algorithm}[H]
	\caption{Non-parametric diffusion-based \textbf{model exploiting} \Comment{Agent $k$}} \label{Alg:2}
	\begin{algorithmic}[1]
	\State \textbf{input:} $x\in\mathcal{D}, \delta, L, \sigma_k$ 
    \State \hskip1.5em $\mathbb{T}^{loc}_{\KK{k}}=\{(\xi_i^{loc},\hat{\mu}_i^{loc},\beta_i^{loc}):i=1,2,\dots\}$ \Comment{Local data set}
   	\State \hskip1.5em $\mathbb{T}^{acq}_{\KK{k}}=\{(\xi_i^{acq},\hat{\mu}_i^{acq},\beta_i^{acq}):i=1,2,\dots\}$\Comment{Acquired data set}
    \State \textbf{output:} $\{\hat{m}_{k,t}(x),\beta_{k,t}(x)\}$
    \State $\sfunction{Evaluate }\beta_{k,t}(x)$ as in Eq.~\eqref{eq:main_bound} 
    \State $\mathbb{I}_{\KK{k}}\gets \{i\colon L\vert x-\xi_i^{acq}\vert<\beta_{k,t}(x)-\beta_i^{acq}\}$
    \If {$\mathbb{I}_{\KK{k}}=\varnothing$},
    		\State $\sfunction{Evaluate }\hat{\mu}_{k,t}(x)$ as in Eq.~\eqref{eq:NW1} 
           \State $\hat{m}_{k,t}(x) \gets \hat{\mu}_{k,t}(x)$
    \Else
           \State $j\gets\argmin_{i\in\mathbb{I}_{\KK{k}}}\{L\vert x-\xi_i^{acq}\vert +\beta_i^{acq}\}$
           \State $\hat{m}_{k,t}(x) \gets \hat{\mu}_{j}^{acq}(x)$
           \State $\beta_{k,t}(x) \gets \beta_{j}^{acq}(x)$
    \EndIf
    \State \textbf{end}
	\end{algorithmic} 
\end{algorithm}

\KK{In line 6 of the model exploiting algorithm, Algorithm~2, a set $\mathbb{I}_k$ of the acquired tuple indices is calculated based on the Lipschitz continuity of $m$ and the idea presented in Theorem 1, that (for a given $x$) provide confidence bounds tighter than for the local estimates. Also, in line 11, if $\mathbb{I}_k$ is nonempty, the tuple with the tightest available bound is found by evaluating $\argmin_{i\in\mathbb{I}_{k}}\{L\vert x-\xi_i^{acq}\vert +\beta_i^{acq}\}$.}

Some non-asymptotic properties of the \CR{proposed} algorithm are established in the following theorem.

\begin{theorem}\label{Th:main}
Let Assumptions 1--4 be in force. Consider any agent $k\in\mathcal{M}$ in the network $\mathcal{G}$ using the tuple exchange protocol of Algorithm 1 and estimating \PW{$m$} as in Algorithm 2. Then, for any $0<\delta<1$, any $t\in\mathbb{N}$ \PW{and every fixed $x\in{\mathcal{D}}$} \CR{(that can be deterministic or in general a function of $\eta_{k,1}^{loc}, \dots, \eta_{k,t-1}^{loc}, \xi_{k,1}^{loc}, \dots, \xi_{k,t}^{loc}$)}, with probability at least $1-\delta$ \CR{(conditioned on $\eta_{k,1}^{loc}, \dots, \eta_{k,t-1}^{loc}, \xi_{k,1}^{loc}, \dots, \xi_{k,t}^{loc}$)},
\begin{align}\label{Eq:th1a}
    \lvert\hat{m}_{k,t}(x)-m(x)\rvert\leq&\min\{\beta_{k,t}(x),B^{acq}(x)\},
\end{align}
where $B^{acq}(x):=\min_i\{B_i^{acq}(x)\}$ and $B_i^{acq}(x)=L\lvert x-\xi_{i}^{acq}\rvert+\beta_{i}^{acq}$.
\end{theorem}
\begin{proof}

According to Algorithm $2$, in case $\beta_{k,t}(x)\leq B^{acq}(x)$, $\mathbb{I}=\varnothing$ and estimator $\hat{m}_{k,t}(x)$ equals $\hat{\mu}_{k,t}(x)$. Hence, inequality \eqref{Eq:th1a} holds at least with probability $1-\delta$ due to Lemma \ref{L:local}. In turn, if $\beta_{k,t}(x)> B^{acq}(x)$, then $\hat{m}_{k,t}(x)=\hat{\mu}_{j}^{acq}$, where $j=\argmin_{i\in\mathbb{I}}\{L\vert x-\xi_i^{acq}\vert +\beta_i^{acq}\}$; this corresponds to the acquired tuple $(\xi_j^{acq},\hat{\mu}_j^{acq},\beta_j^{acq})$, for which $\lvert\hat{\mu}_j^{acq}(\xi_j^{acq})-m(\xi_j^{acq})\rvert\leq \beta_j^{acq}$ with probability $1-\delta$. Therefore, due to Assumption $1$,
\begin{align}
\lvert\hat{\mu}_j^{acq}-m(x)\rvert&=\lvert\hat{\mu}_j^{acq}(\xi_j^{acq})-m(\xi_j^{acq})+m(\xi_j^{acq})-m(x)\rvert \\
&\leq\lvert\hat{\mu}_j^{acq}(\xi_j^{acq})-m(\xi_j^{acq})\rvert+\lvert m(\xi_j^{acq})-m(x)\rvert \\
&\leq\beta_j^{acq}+L\lvert \xi_j^{acq}-x\rvert=B_{j}^{acq}(x)=B^{acq}(x),
\end{align}
with probability at least $1-\delta$, which completes the proof.
\end{proof}

\begin{figure}[ht]
    \centering
    \includegraphics[width=0.6\textwidth]{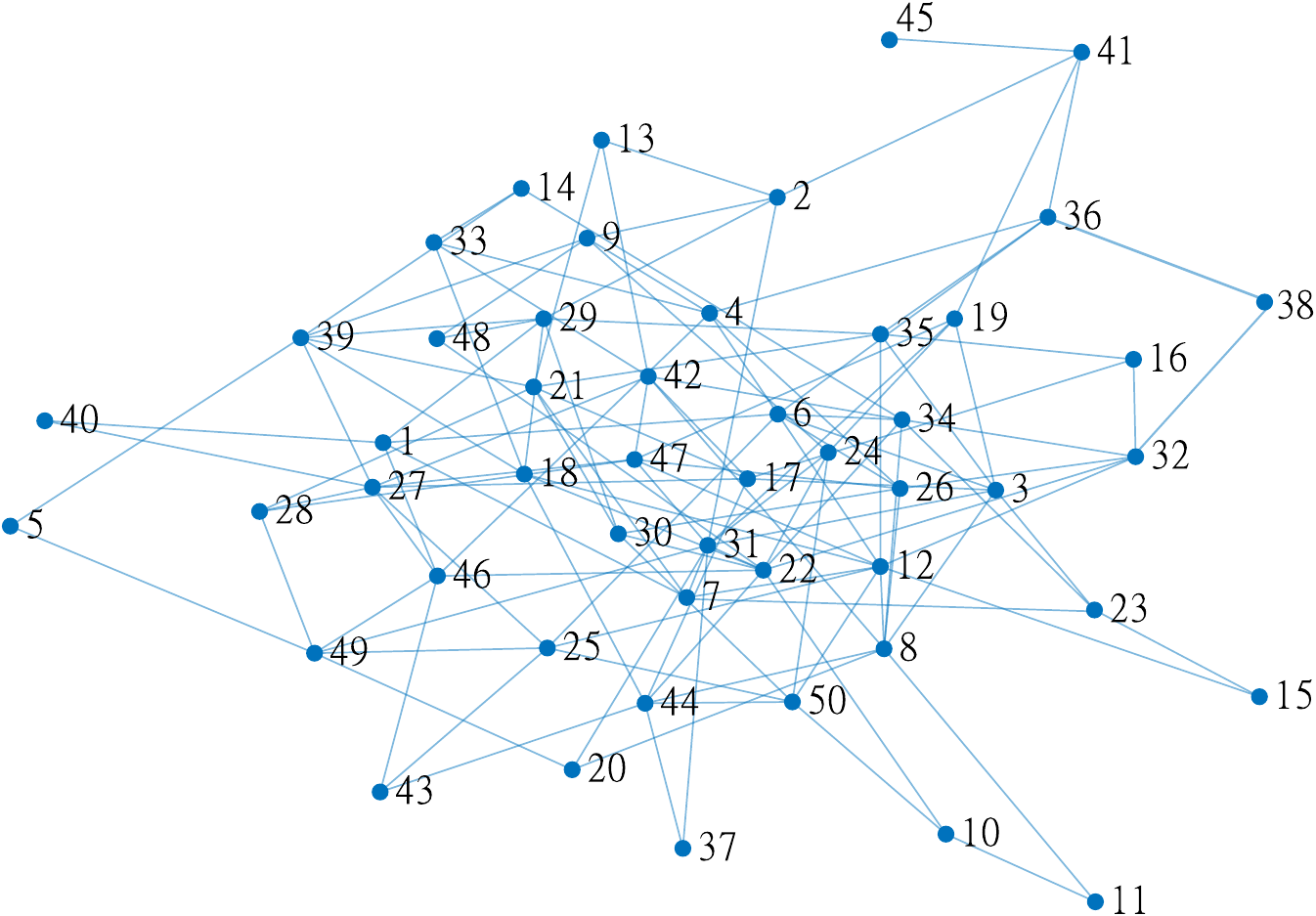}
    \caption{Random topology network with 50 nodes.}
    \label{fig:Net1}
\end{figure}

\section{Numerical experiments}\label{Sec:Num}
In this section, both \textit{model learning} and \textit{model exploiting} algorithms are numerically investigated for the random topology network with 50 nodes shown in Fig.~\ref{fig:Net1}, and the latent non-linearity $m(x)=\sin(x)\exp(-0.2x)+3$. In the experiments\footnote{The Matlab code to obtain the numerical results is available at \url{https://github.com/githuban10/Decentralized-diffusion-based-learning}.}, it is assumed that all agents in the network have local explanatory measurements, $\xi_{k,t}^{loc}$, sampled independently from a normal distribution $\mathcal{N}(\mu_k^{inp},\sigma_k^{inp})$ with randomly selected means $\mu_k^{inp}$ and dispersions $\sigma_k^{inp}$. Similarly, the output noise sequences, $\eta_{k,t}^{loc}$, are generated from independent normal distributions with zero mean and random dispersions, fixed for each agent separately, and sampled uniformly from the interval $(0, 0.7)$. We assume that the known Lipschitz constant $L$ is equal to $1$, the total region of interest, $\mathcal{D}$, is the interval $[0,10]$, and $\delta=0.01$, \ie~all the resulting error bounds should hold with probability~$0.99$ \CR{for each fixed $x \in \mathcal{D}$}.

An example of a histogram of local data collected at node 18 over the time horizon $t=1000$ is shown in Fig.~\ref{fig:node18} (bottom). The same figure contains a histogram of the tuples acquired by the node as a result of the tuple diffusion process. In the experiment, we assumed that all agents in the network generated their tuple requests $\xi_{k,t}^{req}$ purely randomly from a uniform distribution over the entire region of interest $\mathcal{D}=[0,10]$.

\begin{figure}[H]
    \centering
    \includegraphics[width=0.6\textwidth]{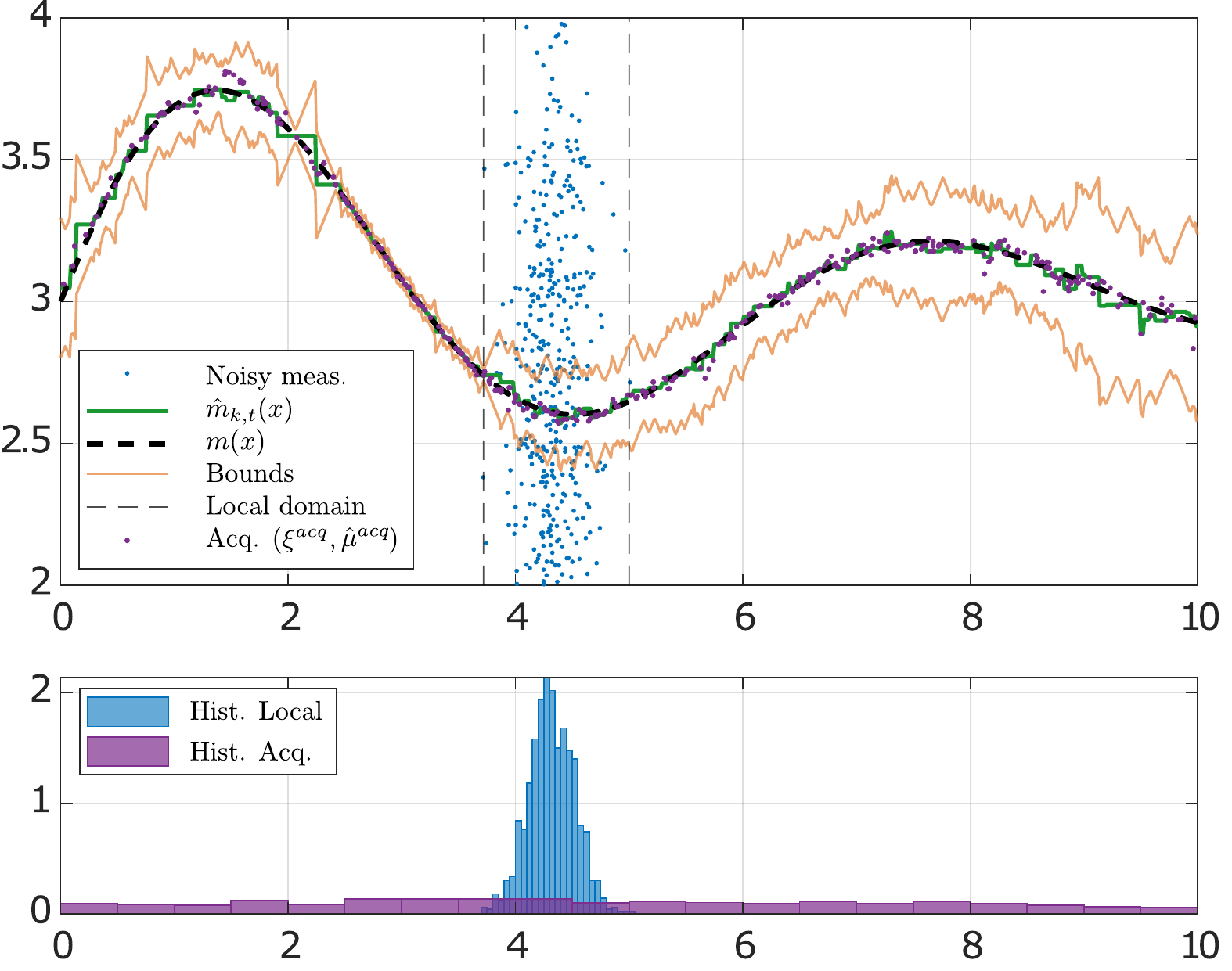}
    \caption{Result of learning for node $k=18$ at time $t=1000$. \textit{Top:} the error bounds (orange) contain the latent phenomenon $m$ (black), although the local measurements of the agent (blue) are far beyond the resulting bounds. \textit{Bottom:} histograms of local measurements (blue) and tuples acquired from the neighbors $\mathcal{N}_k$ (purple).}
    \label{fig:node18}
\end{figure}

To investigate the gain in accuracy due to diffusion-based model learning (Algorithm 1), we compare the outcome of the resulting model, $\hat{m}_{k,t}$, with the inference outcome based on the \textit{local} data only. Due to information diffusion, a single agent is able to properly infer the investigated phenomenon in a domain (much) wider than the support of its own local measurements. This is clearly visible in Fig.~\ref{fig:locest} where we compare the resulting (diffusion-based) model, $\hat{m}_{k,t}$, with the local-data estimation outcome.
\begin{figure}[H]
    \centering
    \includegraphics[width=0.6\textwidth]{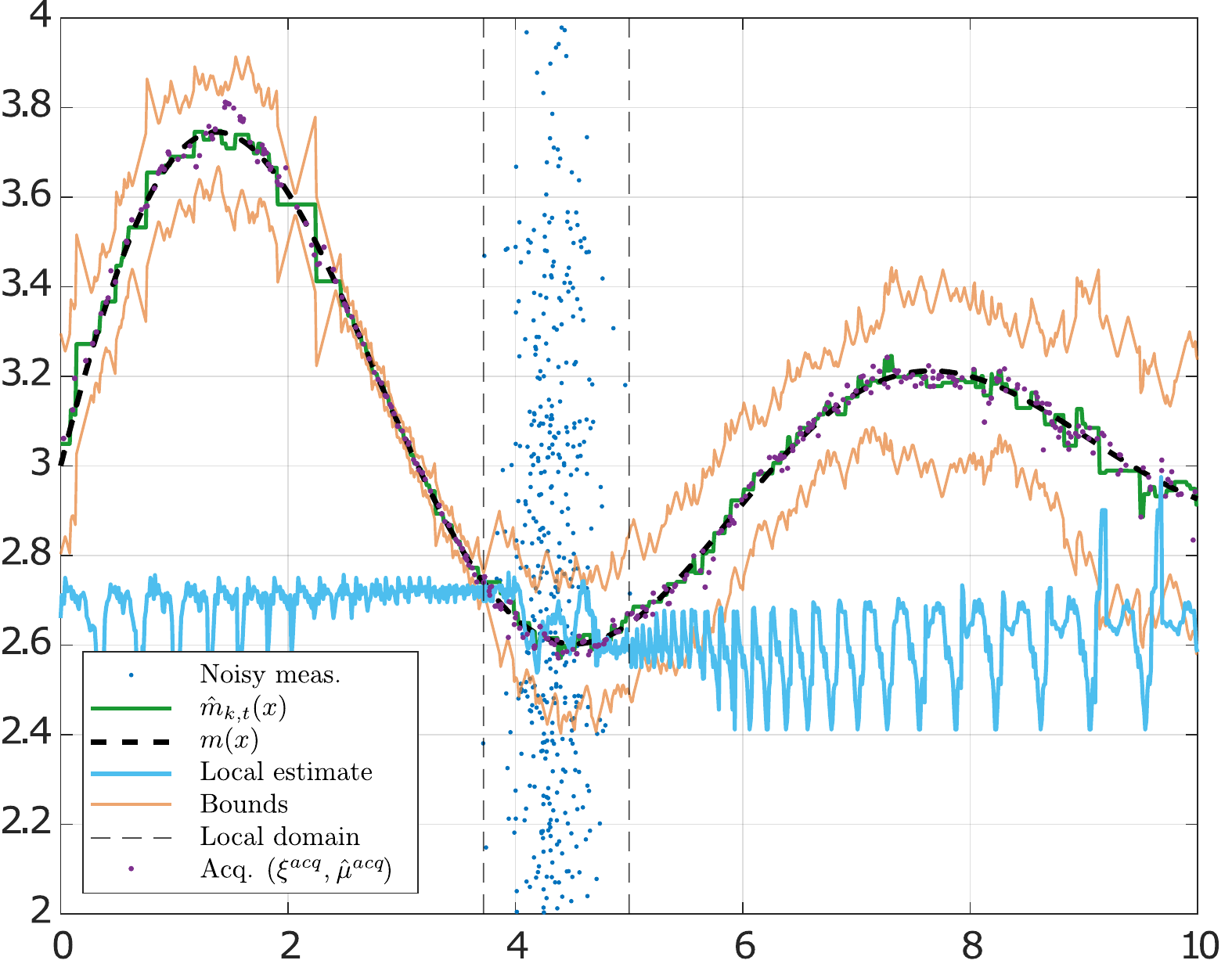}
     \caption{Result of learning for node $k=18$ at time $t=1000$ compared to the estimation based on local, heavily perturbed measurements only (light blue).}
    \label{fig:locest}
\end{figure}

Another interesting property of the proposed method is related to the noise cancellation performed by the agents.  Consider a single agent placed in a highly noisy region of the environment, or equipped with a poor-quality sensor (low signal-to-noise ratio).  Clearly, in that case,  its local measurements are of low quality, which results in low accuracy of the local inference (see Fig.~\ref{fig:betterb}) and non-informative local error bounds. Nevertheless, the acquired neighboring tuples allow for much more accurate estimation results under unchanged confidence level $\delta$. 
\begin{figure}[H]
    \centering
    \includegraphics[width=0.6\textwidth]{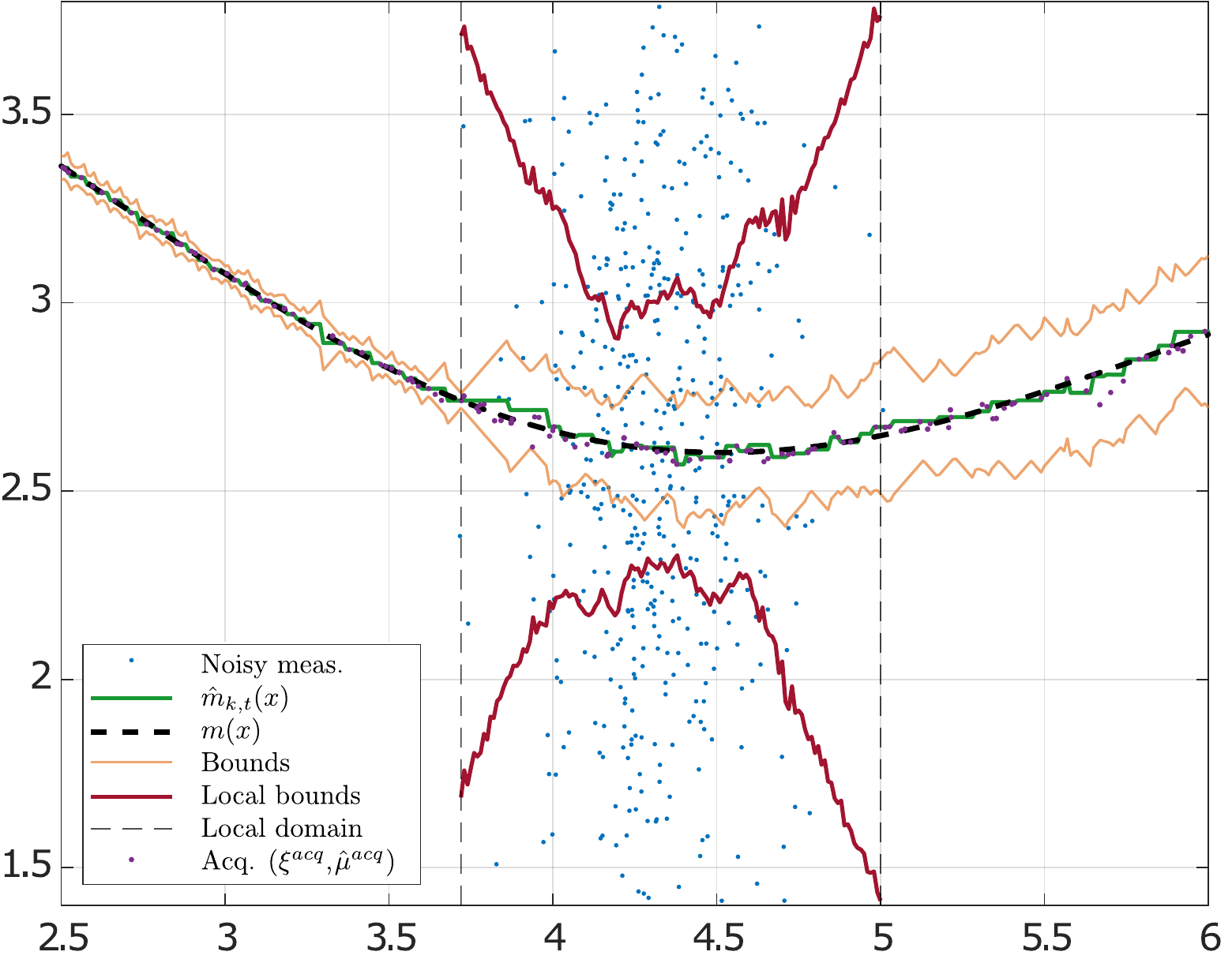}
     \caption{Result of learning for node $k=18$ and $t=1000$ compared to the local agent's bounds (red). }
    \label{fig:betterb}
\end{figure}

In the final experiment, we investigated the evolution in time of the error upper bounds of estimate $\hat{m}_{k,t}$ (for an increasing number of local measurements and acquired tuples). The convergence of the bounds for agents $k=18$ (six connections) and $k=45$ (one connection) is shown in Fig. \ref{fig:bevo} for three different confidence levels $\delta\in\{0.01;0.001;0.0001\}$. Observe that the influence of $\delta$ on the error bounds is relatively minor, which allows for using the algorithm with low $\delta$ values, \ie~with high confidence of the resulting error bounds.
\begin{figure}[H]
    \centering
    \includegraphics[width=0.6\textwidth]{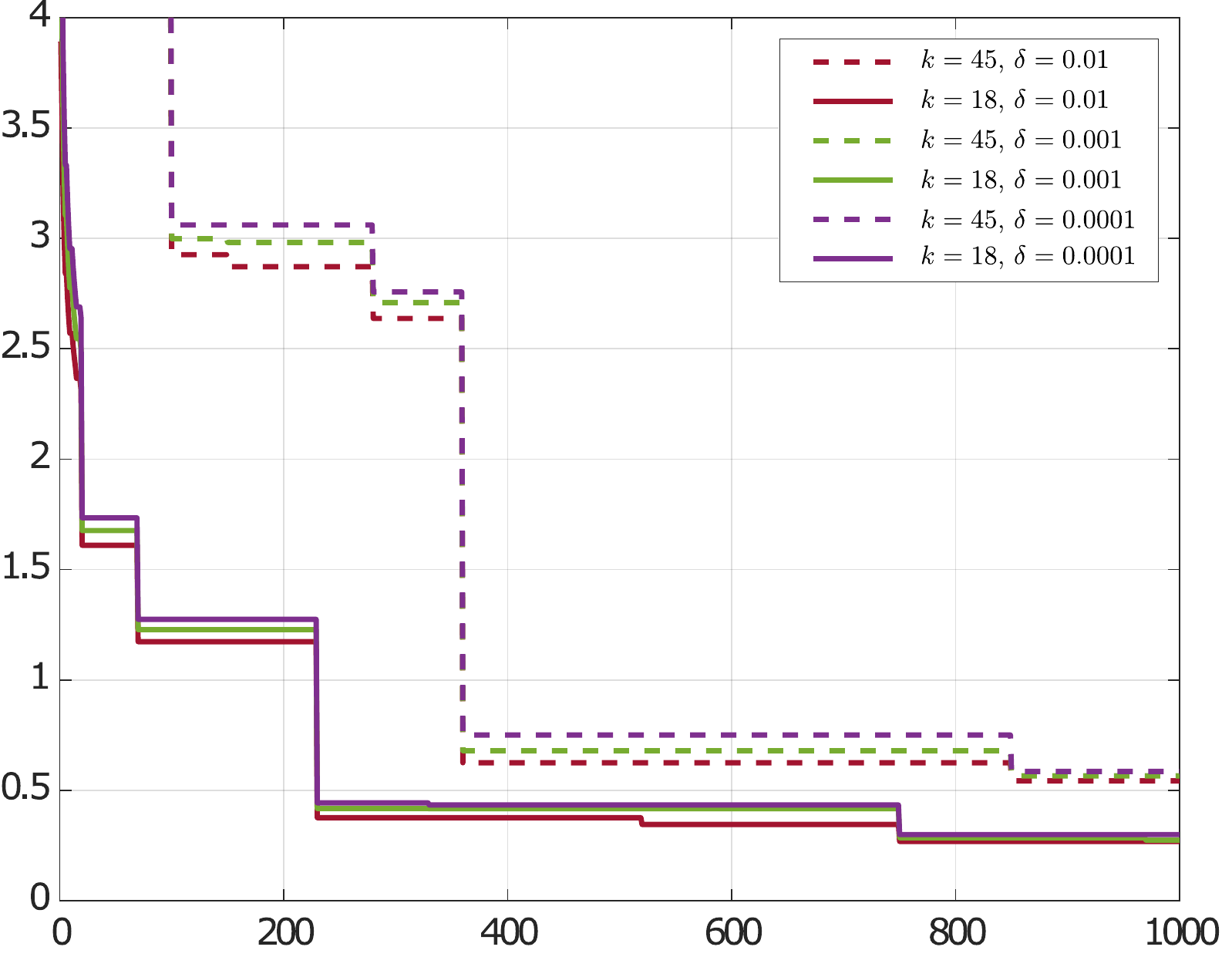}
     \caption{Error upper bounds evolution for nodes $k=18,$ $k=45$ vs. time $t$ for $\delta\in\{10^{-2},10^{-3},10^{-4}\}$.}
    \label{fig:bevo}
\end{figure}
\section{Conclusions}\label{Sec:Con}
In this paper, we have proposed a new non-parametric diffusion-based decentralized learning of Lipschitz-continuous nonlinear phenomena. The proposed algorithm, composed of two cooperating procedures: \textit{model learning} and \textit{model exploiting}, has been formally investigated under mild requirements on the measured signals and additive disturbances. The resulting upper bounds of the estimation error are of non-asymptotic character and hold for user-defined confidence levels. 

We have illustrated the applicability of the method in numerical simulations. According to experimental results, the proposed algorithm ensures collaboration between agents and smooth data diffusion in the network. Due to this, a single agent is able to overcome local obstacles like high noise levels or narrow measurement domain and provide a good-quality model of the investigated phenomenon with known accuracy. Furthermore, the confidence level $\delta$ has a relatively small impact on the derived (non-asymptotic) error bounds, which allows for increasing inference confidence without substantial loss of accuracy.

\section*{Appendix}
\begin{proof}[Proof \textnormal{(}of Lemma \ref{L:local}\textnormal{).}] 
Based on Eq.~\eqref{eq:measurements}, we begin with the observation that, \cf~\cite{dean2021certainty},
\begin{equation}
\left| \sum_{n=1}^{t}\frac{K_{h}( x,\xi _{k,n}^{loc}) }{%
\kappa _{k,t}(x)} y_{k,n}^{loc}-m(x) \right| \leq
\sum_{n=1}^{t}\theta _{n}\vert m( \xi _{k,n}^{loc}) -m(
x) \vert +\left| \sum_{n=1}^{t}\theta _{n}\eta
_{k,n}^{loc}\right|, \label{eq:uppbA}
\end{equation}%
where $\theta _{n}:= K_h( x,\xi _{k,n}^{loc}) /\kappa _{k,t}(x)$. \PW{Note that $\Sigma_{n=1}^t\theta _{n}=1$.}
Due to Assumption $4$, if $K_{h}( x,\xi _{k,n}^{loc}) >0$, then $%
\vert x-\xi _{k,n}^{loc}\vert /h\leq 1$. Therefore, (\cf~Assumption $1$)%
\[
K_{h}( x,\xi _{k,n}^{loc}) >0\quad\Longrightarrow\quad \vert m( \xi_{k,n}^{loc}) -m( x) \vert \leq L\vert x-\xi_{k,n}^{loc}\vert \leq Lh,
\]%
and since the weights $\theta _{n}$ sum up to $1$,%
\[
\sum_{n=1}^{t}\theta _{n}\vert m( \xi _{k,n}^{loc}) -m(
x) \vert \leq Lh.
\]%
For the last term in (\ref{eq:uppbA}), observe that%
\begin{equation}\label{eq:transf}
\left| \sum\nolimits_{n=1}^{t}\theta _{n}\eta _{k,n}^{loc}\right|
=\frac{1}{\kappa _{k,t}(x)} \left| \KK{\sum\nolimits_{n=1}^{t}}K_{h}( x,\xi
_{k,n}^{loc}) \eta _{k,n}^{loc}\right|.
\end{equation}%
According to Lemma \ref{Lem:Tech_A}, the right-hand side of Eq.~\eqref{eq:transf} is upper bounded (with probability $1-\delta$) by
\[
\frac{1}{\kappa_{k,t}(x)} \KK{\sigma_k} \sqrt{2\log \left( \delta ^{-1}\sqrt{1+\sum
\nolimits_{n=1}^{t}K_{h}^{2}( x,\xi _{k,n}^{loc}) }\right) \left(
1+\KK{\sum\nolimits_{n=1}^{t}}K_{h}^{2}( x,\xi _{k,n}^{loc}) \right) }.
\]
Furthermore, since $K_{h}( x,\xi _{k,n}) \leq 1$ (\cf~Assumption 4), we
obtain
\[
\frac{1}{\kappa _{k,t}(x)} \left| \sum\nolimits_{n=1}^{t}K_{h}( x,\xi
_{k,n}^{loc}) \eta _{k,n}^{loc}\right| \leq \KK{\sigma_{k}} \sqrt{2\log
( \delta ^{-1}\sqrt{1+\kappa_{k,t}(x)}) }\frac{\sqrt{1+\kappa_{k,t}(x)%
}}{\kappa _{k,t}(x)}.
\]%
Observe next that, if $\kappa_{k,t}(x) > 1$, then%
\[
\frac{\sqrt{1+\kappa _{k,t}(x)}}{\kappa _{k,t}(x)}<\frac{\sqrt{2\kappa _{k,t}(x)}}{%
\kappa _{k,t}(x)}=\frac{\sqrt{2}}{\sqrt{\kappa _{k,t}(x)}}.
\]%
\KK{Therefore, with probability $1-\delta$, for $\kappa_{k,t}(x) > 1$} 
\begin{equation*}
\frac{1}{\kappa _{k,t}(x)} \left| \sum\nolimits_{n=1}^{t}K_{h}( x,\xi
_{k,n}^{loc}) \eta _{k,n}^{loc}\right| \leq \frac{2 \KK{\sigma_k}}{\kappa _{k,t}(x)} \sqrt{%
\kappa _{k,t}(x) \log \left( \delta ^{-1}\sqrt{1+\kappa _{k,t}(x)}\right)},    
\end{equation*}
whereas for $0<\kappa _{k,t}\leq 1$,%
\begin{align*}
\frac{1}{\kappa _{k,t}(x)} \left| \sum_{n=1}^{t}K_{h}( x,\xi
_{k,n}^{loc}) \eta _{k,n}^{loc}\right| &\leq
\frac{\sigma}{\kappa_{k,t}(x)} \sqrt{2\log\left(\delta ^{-1}\sqrt{1+\kappa_{k,t}(x)}\right) }\sqrt{1+\kappa _{k,t}(x)}\\
&\leq \frac{2\sigma}{\kappa_{k,t}(x)} \sqrt{\log ( \sqrt{2}/\delta)},
\end{align*}
which completes the proof.
\end{proof} 

The following two lemmas are small variations of Theorem $3$ and Lemma $1$ in \cite{abbasi:2011}, respectively.

\begin{lemma}\label{Lem:Tech_A}
Let $\{ v_{t}\colon t\in\mathbb{N} \} $ be a bounded
stochastic process \PW{and $\{ \eta _{t}\colon t\in\mathbb{N}\}$ be a
stochastic process, sub-Gaussian in the sense that there exists a $\sigma >0$ such that for every $\gamma \in \mathbb{R}$, and $t\in\mathbb{N}$,

\begin{equation}\label{subgauss}
\E\{ \exp ( \gamma \eta _{t})|\eta _{1},\dots,\eta_{t-1},v_1,\dots,v_{t} \} \leq \exp \left( \frac{%
\gamma ^{2}\sigma ^{2}}{2}\right).
\end{equation}}
Let also%
\[
S_{t}:=\sum\nolimits_{n=1}^{t}v_{n}\eta _{n}\quad\textnormal{and}\quad V_{t}:=%
\sum\nolimits_{n=1}^{t}v_{n}^{2}.
\]%
Then, for any $t\in\mathbb{N}$ and $0<\delta <1$, with probability $1-\delta$,%
\begin{equation*}
\vert S_{t}\vert \leq \sqrt{2\sigma ^{2}\log \left( \delta ^{-1}%
\sqrt{1+V_{t}} \right) ( 1+V_{t}) }.
\end{equation*}
\end{lemma}
\begin{proof}
Without loss of generality, let $\sigma=1$. For any $\lambda\in \mathbb{R}$ let
\[
w_t(\lambda) :=\exp\left( \lambda S_t - \frac{1}{2}\lambda^2 V_{t}\right).
\]
From Lemma~\ref{Lem:Tech_B}, we note for any $\lambda \in \mathbb{R}$ that $\E\{ w_{t}(\lambda)\} \leq 1$. Let now $\Lambda $ be a $\mathcal{N}( 0,1)$ random variable, independent of all other variables. Clearly, $\E\{ w_{t}(\Lambda) |\Lambda \} \leq 1$. Define%
\[
w_{t}:=\E\{ w_{t}( \Lambda ) |v_n,\eta_n \colon n\in\mathbb{N}\}.
\]%
Then $\E\{ w_{t}\} \leq 1$ since $\E\{ w_{t}\} =\E\{\E\{ w_{t}( \Lambda ) |v_n, \eta_n\colon n\in\mathbb{N} \}\} = \E\{ w_{t}( \Lambda ) \} = \E\{ \E\{
w_{t}(\Lambda) |\Lambda \} \} \leq 1$. We can also express $w_t$ directly as
\begin{align*}
w_{t} &=\frac{1}{\sqrt{2\pi }}\int \exp\left( \lambda S_{t}-\frac{1}{2}%
\lambda ^{2}V_{t}\right) \exp\left( \frac{-\lambda ^{2}}{2} \right) d\lambda 
\\
&= \frac{1}{\sqrt{2\pi }}\int \exp\left( -\frac{1}{2}( V_{t}+1)
\lambda ^{2}+S_{t}\lambda\right) d\lambda,
\end{align*}%
which further gives%
\begin{align*}
w_{t} &=\frac{1}{\sqrt{2\pi }}\int \exp\left( -\frac{1}{2}\frac{(
\lambda -\frac{1}{1+V_{t}}S_{t}) ^{2}}{( 1+V_{t}) ^{-1}}\right) \exp\left( \frac{1}{2}\frac{S_{t}^{2}}{1+V_{t}}\right) d\lambda  \\
&=\frac{1}{( 1+V_{t}) ^{1/2}}\exp\left( \frac{1}{2}\frac{%
S_{t}^{2}}{1+V_{t}} \right) \int \frac{1}{( 1+V_{t}) ^{-1/2}\sqrt{%
2\pi }} \exp\left[ -\frac{1}{2}\left( \frac{\lambda -\frac{1}{1+V_{t}}S_{t}}{%
( 1+V_{t}) ^{-1/2}}\right) ^{2}\right] d\lambda  \\
&=\frac{1}{( 1+V_{t}) ^{1/2}} \exp\left( \frac{1}{2}\frac{%
S_{t}^{2}}{1+V_{t}}\right).
\end{align*}
Therefore $\Pp\{ \delta\, w_{t}\geq 1\} $ is equal to%
\begin{align*}
\Pp\left\{ \frac{\delta }{( 1+V_{t}) ^{1/2}}\exp\left( \frac{1}{2}%
\frac{S_{t}^{2}}{1+V_{t}}\right) \geq 1 \right\}
&= \Pp\left\{ \exp\left( \frac{1}{2}\frac{S_{t}^{2}}{1+V_{t}}\right) \geq ( 1+V_{t}) ^{1/2}%
\frac{1}{\delta } \right\} \\
&= \Pp\left\{ \frac{S_{t}^{2}}{1+V_{t}}\geq 2\log \left[ ( 1+V_{t})^{1/2}\frac{1}{\delta }\right] \right\} \\
&= \Pp\left\{ S_{t}^{2}\geq 2\log \left[ \frac{\sqrt{1+V_{t}}}{\delta }\right]
( 1+V_{t}) \right\}. 
\end{align*}
Recall now that $\E\{w_{t}\}\leq 1$. Hence, due to Markov's inequality,%
\[
\Pp\{ \delta\, w_t \geq 1\} \leq \delta\, \E\{ w_{t}\} \leq \delta,
\]%
which completes the proof.
\end{proof}

\begin{lemma}\label{Lem:Tech_B}
Let $\{ v_{t}\colon t\in\mathbb{N} \} $ and $\{ \eta _{t}\colon t\in\mathbb{N} \} $ be as in Lemma \ref{Lem:Tech_A}. For any $\lambda \in \mathbb{R}$, define%
\begin{equation}
w_{t}( \lambda ) :=\exp \left( \sum_{n=1}^{t}\frac{\lambda \eta
_{n}v_{n}}{\sigma }-\frac{1}{2}\lambda ^{2}v_{n}^{2}\right).
\end{equation}%
Then, $\E\{ w_{t}( \lambda ) \} \leq 1$.
\end{lemma}

\begin{proof}
Let
\[
D_n := \exp\left(\frac{\lambda \eta _{n}v_{n}}{\sigma }-\frac{1}{2}\lambda
^{2}v_{n}^{2}\right).
\]%
Clearly $w_{t}(\lambda) = D_1 D_ 2 \dots D_t$. Note that $\E\{ D_n | \PW{\eta_1,\dots,\eta_{n-1}, v_1,\dots,v_n}\}$ is equal to
\begin{align*}
&\E\left\{ \left. \exp\left( \frac{\lambda \eta_n v_n}{\sigma}\right) / \exp\left( \frac{1}{2} v_n^2\lambda^{2}\right) \right| \PW{\eta_1,\dots,\eta_{n-1}, v_1,\dots,v_n}\right\} \\
&= \E\left\{ \left. \exp\left(\frac{\lambda \eta_n v_n}{\sigma}\right) \right| \PW{\eta_1,\dots,\eta_{n-1}, v_1,\dots,v_n} \right\} \bigg/ \exp\left( \frac{1}{2} v_n^2\lambda^2\right) 
\end{align*}%
Hence, due to \eqref{subgauss}, 
\[
\E\{ D_n|\PW{\eta_1,\dots,\eta_{n-1}, v_1,\dots,v_n}\}
\leq \exp\left( \frac{( \lambda v_n/\sigma )^2 \sigma^2}{2} \right) \bigg/ \exp\left(\frac{1}{2} v_n^2\lambda^2 \right)
= 1.
\]%
Next, for every $t\in\mathbb{N}$,%
\begin{align*}
\E\{ w_t(\lambda) &|\PW{\eta_1,\dots,\eta_{t-1}, v_1,\dots,v_t}\} = \E\{ D_1 \cdots D_{t-1} D_t | \PW{\eta_1,\dots,\eta_{t-1}, v_1,\dots,v_t}\} \\
&= D_1 \cdots D_{t-1} \E\{ D_t|\PW{\eta_1, ... \eta_{t-1}, v_1, ... v_t}\} \leq w_{t-1}(\lambda) .
\end{align*}%
Therefore \PW{$\E\{ w_{t}( \lambda ) \}= \E\{ \E\{ w_{t}(\lambda ) |\eta_1, ... \eta_{t-1}, v_1, ... v_t\} \}$ can be upper bounded by $\E\{ w_{t-1}(\lambda) \}\leq\dots\leq\E\{w_1(\lambda)\}= \E\{ \E\{ D_1|v_1 \}\}\leq 1$, which completes the proof.}
\end{proof}

\bibliographystyle{plain} 
\bibliography{Literature}
\end{document}